\documentclass[sigconf, nonacm]{acmart}

\usepackage{ifthen}
\usepackage{amsmath}
\usepackage{amsfonts}
\usepackage{algorithmic}
\usepackage{graphicx}
\usepackage{textcomp}
\usepackage{xcolor}
\usepackage{xspace}
\usepackage{verbatim}
\usepackage{soul}
\usepackage{url}
\usepackage{hyperref}
\usepackage{subcaption}
\usepackage{caption}
\usepackage{makecell}
\usepackage{float}
\usepackage{longtable}
\usepackage{bm}
\usepackage{pdflscape}
\usepackage[]{mdframed}
\usepackage{multirow}
\usepackage{tikz}
\usepackage{enumitem}
\usepackage{titlesec}
\usetikzlibrary{shapes,math,positioning,arrows.meta}
\usepackage{pgfplots} 
\usepackage{pifont}
\usepackage{booktabs, makecell}
\usepackage[tworuled, linesnumbered, noend]{algorithm2e}
\usepackage{pgfplots}
\pgfplotsset{compat=1.16}
\usepgfplotslibrary{groupplots}
\usetikzlibrary{fit,backgrounds}
\hypersetup{colorlinks=true,urlcolor=blue,citecolor=red}
\SetKw{Continue}{continue}
\SetKw{Break}{break}

\usepackage{graphicx}

\newtheorem{theorem}{Theorem}
\newtheorem{lemma}{Lemma}

\newcommand{\partitle}[1]{\noindent\textbf{#1}\xspace}
\newcommand{\mpole}{MP}

\title{Generative Correlation Manifolds: Generating Synthetic Data with Preserved Higher-Order Correlations}

\author{Jens E. d'Hondt}
\email{j.e.d.hondt@tue.nl}
\affiliation{%
\institution{Eindhoven University of Technology}
\city{Eindhoven}
\country{the Netherlands}
}

\author{Wieger R. Punter}
\email{w.r.punter@tue.nl}
\affiliation{%
\institution{Eindhoven University of Technology}
\city{Eindhoven}
\country{the Netherlands}
}

\author{Odysseas Papapetrou}
\email{o.papapetrou@tue.nl}
\affiliation{%
\institution{Eindhoven University of Technology}
\city{Eindhoven}
\country{the Netherlands}
}

\date{July 3th, 2025}

\begin{document}

\begin{abstract}
    The increasing need for data privacy and the demand for robust machine learning models have fueled the development of synthetic data generation techniques. 
    However, current methods often succeed in replicating simple summary statistics but fail to preserve both the pairwise and higher-order correlation structure of the data that define the complex, multi-variable interactions inherent in real-world systems. 
    This limitation can lead to synthetic data that is superficially realistic but fails when used for sophisticated modeling tasks. 
    In this white paper, we introduce Generative Correlation Manifolds (GCM), a computationally efficient method for generating synthetic data. 
    The technique uses Cholesky decomposition of a target correlation matrix to produce datasets that, by mathematical proof, preserve the entire correlation structure -- from simple pairwise relationships to higher-order interactions -- of the source dataset. 
    We argue that this method provides a new approach to synthetic data generation with potential applications in privacy-preserving data sharing, robust model training, and simulation.
\end{abstract}

\maketitle

\begin{figure}[t]
    \centering
    \begin{tikzpicture}[
  scale=0.9,
  every node/.style={font=\small},
  var/.style={circle, draw, fill=blue!10, minimum size=0.8cm, inner sep=0pt},
  edge/.style={thick, blue!70},
  hedge/.style={draw=none, fill=orange!30, opacity=0.5}
]

\coordinate (leftcenter) at (-2.5,0);
\node[var] (A1) at ($(leftcenter)+(0,1.2)$) {A};
\node[var] (B1) at ($(leftcenter)+(-1,-0.8)$) {B};
\node[var] (C1) at ($(leftcenter)+(1,-0.8)$) {C};

\draw[edge] (A1) -- (B1);
\draw[edge] (B1) -- (C1);
\draw[edge] (A1) -- (C1);

\node at ($(leftcenter)+(0,-2.1)$) {\textbf{Pairwise Correlations (2nd-Order)}};

\coordinate (rightcenter) at (2.5,0);
\node[var] (A2) at ($(rightcenter)+(0,1.2)$) {A};
\node[var] (B2) at ($(rightcenter)+(-1,-0.8)$) {B};
\node[var] (C2) at ($(rightcenter)+(1,-0.8)$) {C};

\filldraw[hedge] (A2.center) -- (B2.center) -- (C2.center) -- cycle;
\draw[thick, orange!70!black] (A2.center) -- (B2.center) -- (C2.center) -- cycle;

\node at ($(rightcenter)+(0,-2.05)$) {\textbf{3rd-Order Correlation}};

\end{tikzpicture}
    \caption{Illustration of the different orders of correlation.}
    \label{fig:order}
\end{figure}

\section{Introduction}
In an era dominated by data-driven discovery, access to high-quality data is paramount~\cite{lu2025machinelearningsyntheticdata,sorscher2023neuralscalinglawsbeating}. 
Yet, this access is often restricted by critical privacy regulations (e.g., GDPR, HIPAA) and the inherent scarcity of data in many specialized domains. 
Synthetic data generation offers a compelling solution, promising to provide statistically representative surrogates without exposing sensitive information or requiring new data collection~\cite{assefa21,bauer2024comprehensiveexplorationsyntheticdata}.

The central challenge, however, lies in the definition of "statistically representative". 
Most generative methods are validated by their ability to match the summary statistics of a source dataset~\cite{sichani24,figueira22}.
While valuable, this is an incomplete measure of accuracy.
Real-world phenomena are rarely governed by univariate statistics or even pairwise relationships; they are driven by a web of intricate, multi-variable dependencies. 
In this work, we focus specifically on higher-order Pearson correlations -- defined as correlations between means of variable groups -- which capture one important aspect of these multi-variable dependencies~\cite{Nguyen2014b,heikinhemo07,knobbe06,carlborg04,mitra16,licher19,chen12,perlin06,agrawal17,agrawal20,santoro23,Zhang2008}.
For example, (a) financial markets display cascading correlation effects during crisis periods~\cite{longin01}, (b) biological systems contain complex gene regulatory networks with multi-order interactions~\cite{Barabasi2004}, and (c) social networks exhibit intricate relationship structures that cannot be captured by simple pairwise correlations~\cite{newman03}.
When synthetic data fails to preserve these structures, downstream applications suffer from reduced model performance, biased statistical analyses, and compromised privacy guarantees.

This white paper introduces a novel approach, \textbf{Generative Correlation Manifolds (GCM)}, that directly addresses this challenge. 
Our method builds upon our previous work on multivariate correlation discovery~\cite{vldbj23} and extends it to synthetic data generation.
Particularly, we prove that preserving pairwise correlations is mathematically sufficient to maintain all higher-order Pearson correlations. 
While other forms of multi-variable dependencies (e.g., mutual information or tensor-based decompositions) may exist, our method focuses specifically on this well-defined correlation structure.
Based on this insight, our method leverages Cholesky decomposition to generate synthetic data that is guaranteed to preserve the complete Pearson correlation hierarchy of a source dataset in a computationally efficient way. 
The implementation of GCM is available as open source software at \url{https://github.com/JdHondt/gcm}.

\section{Related Work}
Current synthetic data generation approaches fall into two main categories: deep learning-based methods and traditional statistical approaches. 
We briefly review these methods and highlight their limitations in capturing complex interactions.

\paragraph{Deep Learning-Based Methods.} 
Synthetic data generation based on state-of-the-art deep learning methods has recently emerged as a promising solution to replace the expensive and laborious collection of real data. Generative Adversarial Networks (GANs) and Variational Autoencoders (VAEs) have shown promise in generating realistic synthetic datasets~\cite{goodfellow2014generativeadversarialnetworks,kingma2022autoencodingvariationalbayes}. 
However, these methods offer no mathematical guarantees regarding the preservation of relationships between features, particularly higher-order correlations.
More recently, transformer-based architectures like GPT and BERT variants have been adapted for synthetic data generation~\cite{brown2020languagemodels}. While these models excel at capturing sequential patterns and contextual relationships, they similarly lack explicit guarantees for preserving correlation structures across features.

\paragraph{Statistical Approaches.} 
Traditional statistical methods, including copula-based techniques, have been employed for synthetic data generation. 
Copulas provide a flexible framework for modeling multivariate distributions by separating the marginal distributions from their dependency structure~\cite{nelsen2007introduction}. 
While some copula methods can theoretically preserve the complete correlation structure through accurate modeling of pairwise dependencies, they often become computationally intractable for higher dimensions and require complex parameter estimation. Other approaches like SMOTE and ADASYN focus primarily on class balance rather than preserving feature relationships~\cite{smote}.

\paragraph{Our Contribution}
Our method's key contribution lies in its elegant simplicity and mathematical proof that preserving pairwise correlations is sufficient to maintain all higher-order relationships. 
While some existing methods may achieve similar preservation indirectly, GCM provides a direct, computationally efficient approach through its manifold-based transformation. 
This mathematical insight allows us to guarantee the preservation of the complete correlation structure while avoiding the complexity of explicitly modeling higher-order dependencies.

\section{The Challenge of Capturing Higher-Order Correlations}
To understand the importance of this work, it is crucial to distinguish between different orders of correlation.

\partitle{Pairwise (2nd-Order) Correlation}:
This is the similarity between two variables. 
For example, an increase in marketing spend is correlated with an increase in sales. 
The most common measure of correlation is Pearson's correlation coefficient ($\rho$), defined for two variables $x$ and $y$ as:
\begin{equation}
    \rho(x, y) = \frac{\sum_{i=1}^n {(x_i - \bar{x})}{(y_i - \bar{y})}}{\sqrt{\sum_{i=1}^n {(x_i - \bar{x})}^2}\sqrt{\sum_{i=1}^n {(y_i - \bar{y})}^2}}
\end{equation}
where $\bar{x}$ and $\bar{y}$ are the means of $x$ and $y$ respectively. This coefficient ranges from -1 (perfect negative correlation) to 1 (perfect positive correlation).

\partitle{Higher-Order Correlation:} 
This involves the interaction between three or more variables. 
For example, consider a medical dataset where a specific gene (A), a particular lifestyle factor (B), and a negative health outcome (C) are studied. 
A model looking only at pairwise correlations might find weak links between A-C and B-C. 
However, the true risk might only become significant when both A and B are present simultaneously. 
This three-way interaction is a higher-order correlation. 
The formalization and discovery of such multivariate dependencies is a key topic in data mining research~\cite{vldbj23}.

A common way to quantify these higher-order relationships is through the \textit{Multipole} correlation measure (also known as the $k$-th order correlation with $k>2$), which extends the standard correlation coefficient to multiple variables~\cite{vldbj23,agrawal20}. 
The multipole correlation $\mpole(X)$ measures the linear dependence of an input set of features $X$~\cite{agrawal20}.
Specifically, let $\hat{x}_1, \ldots, \hat{x}_n$ denote $n$ z-normalized input (column) features, and $\vec{X} = [\hat{x}_1,\ldots, \hat{x}_n]$ the matrix formed by concatenating the features.
Then:
\begin{equation}\label{eqn:mpole_definition}
\mpole(X) = 1 - \underset{\vec{v}\in \mathbb{R}^n, \hat{\vec{v}} = 1}{\min}\text{var}(\mathbf{X} \cdot \vec{v}^T)
\end{equation}
The value of $\mpole(X)$ lies between 0 and 1. The measure takes its maximum value when there exists perfect linear dependence, meaning that there exists a vector $\mathbf{v}$ with norm 1, such that $\text{var}(\mathbf{X} \cdot \vec{v}^T)=0$.

Machine learning models, particularly deep neural networks and complex ensemble methods, can implicitly learn and exploit these higher-order structures to achieve state-of-the-art performance~\cite{wang17}. 
However, when a model is trained on synthetic data that lacks this structural richness, it learns a flawed representation of the problem space, leading to poor generalization and unreliable performance on real-world data.
Therefore, synthetic data generation methods that are evaluated solely on their ability to replicate simple aggregate statistics (means, variances) may appear successful while failing to capture the complex, higher-order relationships that are critical for effective modeling. 
This oversight can lead to significant gaps between the performance of models trained on synthetic versus real data, particularly in domains where multi-variable interactions drive key outcomes.

\section{Methodology: The GCM Method}
The GCM method is elegant in its simplicity and powerful in its mathematical guarantees.

\paragraph{Problem Formulation}
Let $D \in \mathbb{R}^{m_D \times n}$ be a dataset of $n$ features with $m_D$ dimensions each, with column-wise means $\mu_D \in \mathbb{R}^n$ and standard deviations $\sigma_D \in \mathbb{R}^n$. 
Define $C_{D,k} \in \mathbb{R}^{n^k}$ as the $k$-th order correlation matrix of $D$. 
This formulation extends our previous work on multivariate correlation discovery~\cite{vldbj23}, which established the theoretical foundation for identifying and measuring higher-order correlations in both static and streaming data contexts.

\paragraph{Objective} 
Generate a synthetic dataset $S$ with $n$ features and $m_S$ dimensions such that its correlation structure matches that of $D$, i.e., $C_{S,k} = C_{D,k}$ for all $k \leq m_S$ and the mean and variance of each feature in $S$ matches that of $D$, i.e., $\mu_S = \mu_D$ and $\sigma_S = \sigma_D$.

\paragraph{Intuition}
Instead of attempting to learn a complex data distribution from scratch, our method begins with the data's relational blueprint: its \textit{pairwise correlation matrix}. 
We conceptualize this matrix as defining a specific ``shape'' or ``manifold'' in a high-dimensional space. 
Any dataset conforming to this manifold will share the same fundamental relational properties.
The GCM method uses a well-established linear algebra technique, Cholesky decomposition, as a transform. 
It takes unstructured, random noise and projects it onto this predefined correlation manifold. 
The result is a synthetic dataset that perfectly embodies the target correlation structure while preserving the original mean and variance properties of each feature.

\paragraph{Theoretical Foundation}
The foundational discovery behind GCM is that all higher-order correlations are deterministic functions of the pairwise correlation matrix. 
This is a non-obvious but provable property.
This means that if we can perfectly replicate the pairwise correlation structure, we inherently and automatically replicate the entire higher-order correlation structure for free. 
The method does not approximate these complex relationships; it reconstructs them exactly.
For non-normalized data (i.e., $\mu_D \neq 0$ or $\sigma_D \neq 1$), we can preserve both the correlation structure and the original feature statistics by applying appropriate mean and variance transformations after generating the correlated data.
Particularly, our approach is built upon the following key theorem:
\begin{theorem}
    Let $D \in \mathbb{R}^{m_D \times n}$ be a dataset with column-wise means $\mu_D \in \mathbb{R}^n$, standard deviations $\sigma_D \in \mathbb{R}^n$, and correlation matrix $C_{D,2}$. 
    A synthetic dataset $S$ with $n$ features and $m_S$ dimensions generated through Cholesky decomposition of $C_{D,2}$, followed by appropriate mean and variance transformations, is guaranteed to have the same $k$-th order correlation structure as $D$, i.e., $C_{S,k} = C_{D,k}$ for all $k \leq m_S$.
\end{theorem}
\begin{proof}
    The proof relies on demonstrating that higher-order correlations can be expressed as functions of pairwise correlations. 
    This builds upon our foundational work~\cite{vldbj23} which showed that multivariate correlations in static data can be decomposed into constituent pairwise relationships. 
    By preserving the pairwise correlation structure exactly, all higher-order structures are automatically preserved.
    Since Pearson correlation is invariant to linear transformations (scaling and shifting), the mean and variance transformations applied to match the original data's feature statistics do not affect the correlation structure.
    The detailed proof is provided in Appendix~\ref{app:proof}.
\end{proof}


\paragraph{Process}
The generation process is highly efficient and builds upon standard statistical methods for generating correlated variates~\cite{gentle2003random,rubinstein2016simulation};
\begin{enumerate}[leftmargin=*, itemsep=0pt]
    \item \textit{Extract Statistics}: Given a source dataset $D$, compute its column-wise means $\mu_D$ and standard deviations $\sigma_D$.
    \item \textit{Extract Blueprint}: Compute the $n \times n$ pairwise correlation matrix $C$ from $D$.
    \item \textit{Decompose}: Perform a Cholesky decomposition on $C$ to obtain the lower triangular matrix $L$, where $C = LL^T$.
    \item \textit{Generate}: Create a matrix $Z$ of independent random variables drawn from a standard normal distribution with dimensions $m_S \times n$.
    \item \textit{Transform}: Compute the intermediate dataset $\hat{S} = ZL$.
    \item \textit{Denormalize}: For each column $i$, compute $S_i = \hat{S}_i \cdot \sigma_{D,i} + \mu_{D,i}$ to obtain the final synthetic dataset $S$.
\end{enumerate}

The resulting synthetic dataset $S$ is guaranteed to have a pairwise correlation matrix identical to $C$, and therefore, an identical higher-order correlation structure to the original dataset, while matching the mean and variance of each feature.

\paragraph{Computational Complexity}
The algorithm requires $O(n^3)$ operations for the Cholesky decomposition and $O(m_S * n^2)$ for data generation. While these complexities are non-trivial for very large datasets, the method has the advantage of being non-iterative, requiring only a single pass to generate the synthetic data once the correlation matrix is computed.

\section{Use Cases and Applications}
The ability to generate data with such high structural fidelity unlocks numerous possibilities:

\begin{itemize}[leftmargin=*, itemsep=0pt]
    \item \textbf{Privacy-Preserving Data Sharing:} 
    Distribute synthetic datasets that retain the full statistical utility of private source data, allowing external researchers to conduct complex modeling without ever accessing sensitive records~\cite{fung10}.
    \item \textbf{Robust Model Augmentation:} 
    Augment small or imbalanced datasets to improve the training, generalization, and fairness of machine learning models, particularly in fields like finance and medicine where feature interactions are critical~\cite{chen24}.
    \item \textbf{High-Fidelity Simulation}: 
    Create realistic, multi-variate inputs for complex systems modeling, such as financial market stress tests, epidemiological forecasting, and climate change simulations~\cite{glasserman2004monte}.
    \item \textbf{Algorithmic Fairness and Auditing:} 
    Generate controlled datasets with specific correlation structures to systematically test machine learning models for bias arising from complex interactions between sensitive attributes and other features~\cite{calders10}.
\end{itemize}

\section{Call for Collaboration}
The work presented here establishes the theoretical foundation of Generative Correlation Manifolds. 
We believe this is the first step toward a new class of synthetic data generation tools and are actively seeking collaborators to explore its potential.

We are particularly interested in pursuing research in the following areas:
\begin{itemize}[leftmargin=*, itemsep=0pt]
    \item \textbf{Beyond Correlation:} Investigating the preservation of other types of higher-order correlations, such as mutual information and non-linear relationships, to further enhance the method's applicability~\cite{vldbj23}.
    \item \textbf{Domain-Specific Applications:} Applying GCM to pressing challenges in fields like genomics, climate science, social sciences, neuroimaging, and finance, where complex multi-variable interactions are common~\cite{santoro23}
    \item \textbf{Scalability and Performance:} Benchmarking the method on extremely high-dimensional datasets and optimizing its computational performance.
\end{itemize}
If you or your organization are working on challenges related to synthetic data, data privacy, or robust modeling, we invite you to connect with us.

\section{Conclusion}
We have presented the Generative Correlation Manifold method as a potential new approach to synthetic data generation. 
The method's focus on preserving the complete correlation structure of datasets offers promising opportunities for creating representative synthetic data. 
While further research is needed to fully understand its capabilities and limitations, initial results suggest that GCM could contribute to advancing the field of synthetic data generation, particularly in applications where preserving complex data relationships is essential.

\bibliographystyle{ACM-Reference-Format}
\bibliography{references}

\newpage
\appendix
\section{Formal Proof: Generation of Data with Specific Higher-order Correlation Structure via Cholesky Decomposition}\label{app:proof}
\subsection{Theorem}
Let $D \in \mathbb{R}^{m_D \times n}$ be a dataset of $n$ features with $m_D$ dimensions each, with column-wise means $\mu_D \in \mathbb{R}^n$ and standard deviations $\sigma_D \in \mathbb{R}^n$.
Let $C_{D,k} \in \mathbb{R}^{n^k}$ be the $k$-th order correlation matrix (symmetric positive semi-definite with ones on the diagonal) of $D$. 
Then, a synthetic dataset $S$ with $n$ features and $m_S$ dimensions generated through Cholesky decomposition of $C_{D,2}$ followed by appropriate mean and variance transformations is guaranteed to have the same $k$-th order correlation structure as $D$, i.e., $C_{S,k} = C_{D,k}$ for all $k \leq m_S$.

\subsection{Definitions}
\begin{itemize}[leftmargin=*, itemsep=0pt]
\item Let the Multipole correlation be defined as in Equation~\ref{eqn:mpole_definition}.
\item Let z-normalization of a vector $x$ be defined as: $\hat{x} = \frac{x - \bar{x}}{\sigma_x}$, where $\bar{x}$ is the mean and $\sigma_x$ is the standard deviation of $x$.
\item Let $Z \in \mathbb{R}^{m \times n}$ be a matrix whose rows are independent random vectors $z_i \in \mathbb{R}^n$ with $\mathbb{E}[z_i] = 0$ and $\text{Cov}(z_i) = I_n$.
\item Let $X = ZL$ where $L$ is the Cholesky factor of $C$.
\item Let $\hat{C}_{X,k}$ be the sample correlation matrix computed from the $m_X$ samples in $X$.
\end{itemize}

\subsection{Lemmas}
\begin{lemma}[K-order Correlation as a Function of Pairwise Correlations]\label{lem:mv2pairwise}
    The multipole correlation $\mpole(X)$ of a set of vectors $X = [x_1, ..., x_k]$ can be rewritten as~\cite{agrawal20}:
    \begin{align}
        \mpole(X) = 1 - \lambda_{\min}(C_{X,2})
    \end{align}
    where $\lambda_{\min}(C_{X,2})$ is the smallest eigenvalue of the second-order correlation matrix $C_{X,2}$.
    \begin{proof}
        We refer to~\cite{agrawal20} for the proof of this lemma.
    \end{proof}
\end{lemma}

\subsection{Main Proof}
We need to show that the $k$-th order correlation of the generated dataset $S$ is equal to the $k$-th order correlation of the original dataset $D$.

\subsubsection{Step 1: Generating Data with a Given Pairwise Correlation Matrix}
Given a target correlation matrix $C_{D,2} \in \mathbb{R}^{n \times n}$, we can generate a synthetic dataset $\hat{S} \in \mathbb{R}^{m_S \times n}$ with the desired correlation structure using the Cholesky decomposition method. The procedure is as follows:

\begin{enumerate}
    \item Compute the Cholesky decomposition of the correlation matrix $C_{D,2} = LL^T$, where $L$ is a lower triangular matrix.
    \item Generate a matrix $Z \in \mathbb{R}^{m_S \times n}$ of independent random variables with standard normal distribution.
    \item Compute $\hat{S} = ZL$.
\end{enumerate}

The resulting matrix $\hat{S}$ will have the correlation structure specified by $C_{D,2}$ as $m_S$ approaches infinity. This is because the expected correlation matrix of $\hat{S}$ is:

\begin{align}
    E[\hat{S}^T\hat{S}] &= E[(ZL)^T(ZL)] \\
    &= E[L^TZ^TZL] \\
    &= L^T E[Z^TZ] L \\
    &= L^T I L \\
    &= L^T L \\
    &= C_{D,2}
\end{align}

This approach is well-established in the statistical literature~\cite{rubinstein2016simulation,gentle2003random} and provides a direct method for generating data with a specified correlation structure.

\subsubsection{Step 2: Denormalization to Match Original Statistics}
To generate synthetic data that matches the mean and variance of the original (potentially non-normalized) dataset $D$, we apply a denormalization transformation to $\hat{S}$:

For each column $i$ in $\hat{S}$, compute:
\begin{align}
    S_i = \hat{S}_i \cdot \sigma_{D,i} + \mu_{D,i}
\end{align}

where $\mu_{D,i}$ and $\sigma_{D,i}$ are the mean and standard deviation of the $i$-th column in the original dataset $D$.

This transformation is a linear transformation of the form $y_i = a_i x_i + b_i$ where $a_i = \sigma_{D,i}$ and $b_i = \mu_{D,i}$.
Since Pearson correlation is invariant to such linear transformations with $a_i > 0$, we have:
\begin{align}
    \rho(S_i, S_j) = \rho(\hat{S}_i, \hat{S}_j)
\end{align}

Therefore, the final synthetic dataset $S$ maintains the pairwise correlation structure: $C_{S,2} = C_{D,2}$, while having the same feature-wise means and standard deviations as the original dataset $D$.

\subsubsection{Step 3: Preserving Higher-order Correlation Structure}
To complete our proof, we need to show that the higher-order correlation structure is also preserved, i.e., $C_{S,k} = C_{D,k}$ for all $k \leq m_S$.

By Lemma~\ref{lem:mv2pairwise}, the $k$-th order correlation between any two sets of features $\hat{X} = [\hat{x}_1, ..., \hat{x}_s]$ and $\hat{Y} = [\hat{y}_1, ..., \hat{y}_t]$ with $s+t=k$ can be expressed solely in terms of their pairwise correlations:

\begin{equation}
    \mpole(\hat{X}) = 1 - \lambda_{\min}(C_{X,2})
\end{equation}

Since we have established that $S$ has the same pairwise correlation structure as $D$ (i.e., $\rho(\hat{s}_i, \hat{s}_j) = \rho(\hat{d}_i, \hat{d}_j)$ for all columns $i$ and $j$), the $k$-th order correlation between any two sets of features will also be identical.

For any combination of $k$ features from $S$, the $k$-th order correlation will be computed using the same pairwise correlations as the corresponding features in $D$. 
Therefore, by the formula in Lemma~\ref{lem:mv2pairwise}, we can conclude that the $k$-th order correlation structures are identical, i.e., $C_{S,k} = C_{D,k}$ for all $k \leq m_S$.

\end{document}